\documentclass[10pt,twocolumn,letterpaper]{article}

\usepackage[pagenumbers]{cvpr} 

\definecolor{cvprblue}{rgb}{0.21,0.49,0.74}
\usepackage[pagebackref,breaklinks,colorlinks,allcolors=cvprblue]{hyperref}

\usepackage[table]{xcolor}
\usepackage{multirow}

\usepackage{amsmath, amssymb, amsthm}
\usepackage{graphicx}

\newtheorem{theorem}{Theorem}
\newtheorem{assumption}{Assumption}
\newtheorem{corollary}{Corollary}

\newcommand{\specialcell}[2][c]{%
  \begin{tabular}[#1]{@{}c@{}}#2\end{tabular}}

\def\paperID{****} 
\def\confName{CVPR}
\def\confYear{2026}

\title{L3DR: 3D-aware LiDAR Diffusion and Rectification}

\author{Quan Liu$^1$ \ \ \ \ \ Xiaoqin Zhang$^{2}$ \ \ \ \ \ Ling Shao$^3$ \ \ \ \ \ Shijian Lu$^{1,\dagger}$ \\
$^1$Nanyang Technological University \ \ \ \ \ $^2$Zhejiang University of Technology\\
$^3$UCAS-Terminus AI Lab, University of Chinese Academy of Sciences\\
{\href{https://github.com/liuQuan98/L3DR}{https://github.com/liuQuan98/L3DR}}
}

\begin{document}

\twocolumn[{%
\renewcommand\twocolumn[1][]{#1}%
\maketitle
\begin{center}
    \centering
    \captionsetup{type=figure}
    \vspace{-0.6cm}
    \includegraphics[width=\linewidth]{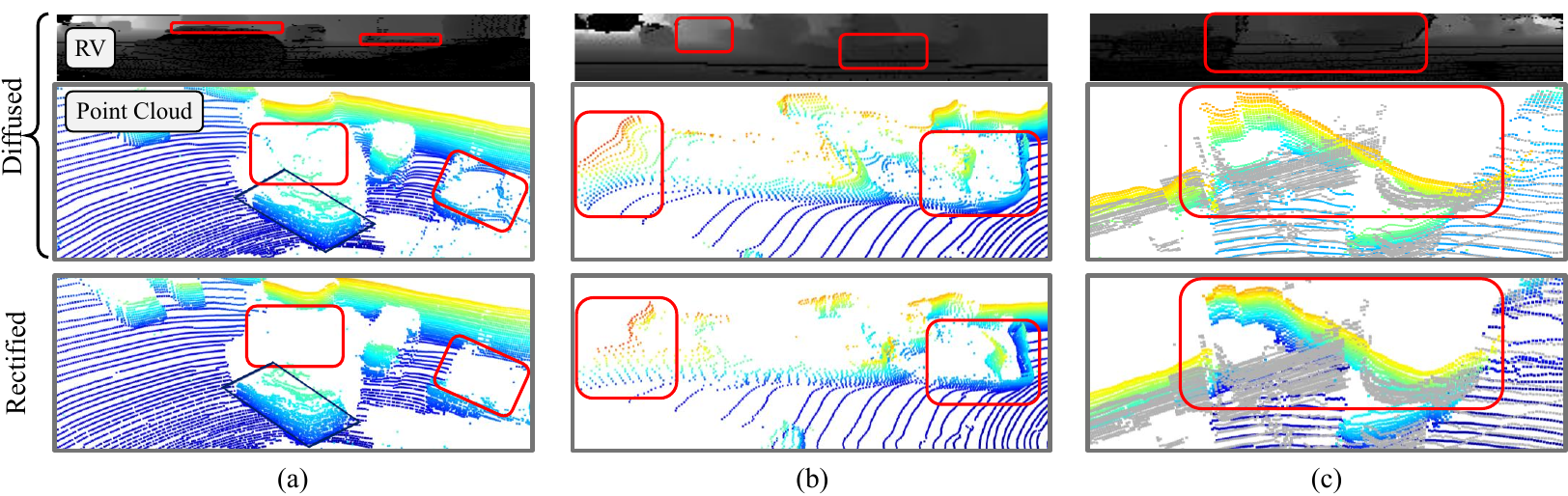}
    \vspace{-0.7cm}
    \captionof{figure}{\textbf{L3DR effectively rectifies LiDAR range-view (RV) diffusion artifacts by selectively ignoring anomalous training regions.} (a) \textit{Depth bleeding} creates fake points between the foreground vehicle and the background. (b) \textit{Wavy surfaces and rounded edges} synthesized by RV diffusion are straightened and sharpened after rectification. (c) Anomalous regions in training data pairs, \textit{e.g.,} a diffusion-generated wall perpendicular to ground truth (GT), can overshadow RV artifacts and hijack artifact removal task; these are suppressed with the Welsch Loss (see section \ref{sec:decomposition_of_objective}). Generated and rectified point clouds are \textcolor{orange!80!black}{c}\textcolor{yellow!85!black}{o}\textcolor{green!80!black}{l}\textcolor{cyan!90!black}{or}\textcolor{blue}{ed} while GT point clouds are in \textcolor{gray}{gray}.}
    \label{fig:motivation}
    \vspace{0.4cm}
\end{center}%
}]

\begin{abstract}
Range-view (RV) based LiDAR diffusion has recently made huge strides towards 2D photo-realism. However, it neglects 3D geometry realism and often generates various RV artifacts such as depth bleeding and wavy surfaces. We design L3DR, a 3D-aware LiDAR Diffusion and Rectification framework that can regress and cancel RV artifacts in 3D space and restore local geometry accurately. Our theoretical and empirical analysis reveals that 3D models are inherently superior to 2D models in generating sharp and authentic boundaries. Leveraging such analysis, we design a 3D residual regression network that rectifies RV artifacts and achieves superb geometry realism by predicting point-level offsets in 3D space. On top of that, we design a Welsch Loss that helps focus on local geometry and ignore anomalous regions effectively. Extensive experiments over multiple benchmarks including KITTI, KITTI360, nuScenes and Waymo show that the proposed L3DR achieves state-of-the-art generation and superior geometry-realism consistently. In addition, L3DR is generally applicable to different LiDAR diffusion models with little computational overhead.
\end{abstract}
\let\thefootnote\relax
\footnotetext{\textsuperscript{\textdagger} Corresponding author: \textcolor{cvprblue}{Shijian.Lu@ntu.edu.sg}}
\vspace{-1.3cm}
\section{Introduction}
Point clouds captured by LiDARs (Light Detection and Ranging) are the cornerstone of outdoor 3D computer vision, facilitating various perception tasks such as detection \citep{zhu2022vpfnet, ku2018joint}, segmentation \citep{xu2020squeezesegv3}, tracking \citep{wu20213d,simon2019complexer}, and SLAM \citep{montemerlo2002fastslam,mur2017orb}. Though they offer superior perception capabilities and safety guarantees for self-driving vehicles, collecting large-scale LiDAR point clouds is prohibitive in terms of sensor procurement, intensive human labour involved~\citep{Caesar_2020_CVPR,Sun_2020_CVPR}, etc. Automated generation of high-quality point clouds has thus become a critical factor for extensive adoption of LiDAR data in various 3D perception tasks.

To this end, a practical LiDAR point cloud generation method must achieve 1) realism in global layout that accurately recreates the coarse-grained global depth distribution on the range view (RV) representation, 2) realism in local geometry that perfectly replicates fine-grained local geometries such as self-occlusion, and 3) light computational overhead which empowers cost-efficient data generation.


Inspired by the massive success of Diffusion Models (DMs) \citep{sohl2015deep} featured by Stable-diffusion \citep{rombach2022high} and ControlNet \citep{zhang2023adding} in image generation, a line of research pioneered by LiDARGen \citep{zyrianov2022learning} has extended image-based DMs for generating layout-realistic point clouds in the RV representation. However, although RV enables DM-based point cloud generation by projecting 3D point clouds to 2D images, it hinders accurate discernment of sparsity and self-occlusion in the 3D space which directly leads to various range-view artifacts such as \textit{depth bleeding} and \textit{wavy surfaces}.
As depicted in Figure \ref{fig:motivation}(a), even realistic RV images can have serious depth bleeding artifacts, a phenomenon of errorneous depth continuity near edges. Additionally, since 3D plain surfaces are projected onto sinusoidal surfaces in images, LiDAR DMs tend to produce suboptimal 3D geometry-realism with wavy surfaces and round corners as illustrated in Figure \ref{fig:motivation} (b). These artifacts greatly undermine the realism of 3D data generated by LiDAR diffusion models. 

This paper presents L3DR, a 3D-aware LiDAR Diffusion and Rectification framework that generates geometrically realistic point clouds by rectifying the back-projected 3D coordinates of diffusion-generated RV images. L3DR works by tackling two challenges. The first is on training data for training a residual regression network (RRN) to rectify 3D coordinates. Unlike normal diffusion training data that can be obtained by adding noises to the input data, RV artifacts exhibit irregular occurrence with little explicit distributions. We therefore introduce semantic-conditioned LiDAR diffusion to generate point clouds and their ground-truth (GT). Thanks to the powerful conditional diffusion models, the obtained data are structurally similar yet imbued with RV artifacts, making them perfect for training the regression network. The second is on training loss. Due to anomalous regions scattered across the training data as illustrated in Figure \ref{fig:motivation} (c), minimizing residual regression errors with L1/L2 loss may over-attend to the anomalous regions and neglect local geometries. We thus introduce Welsch Loss to guide the network to ignore high-bias regions and focus on local geometry artifacts, leading to superior geometry realism in the generated point clouds.

The contributions of this work can be summarized in three major aspects:

\begin{itemize}
    \item We propose a 3D-aware LiDAR Diffusion and Rectification framework that rectifies RV geometry artifacts with a 3D residual regression network and achieves superb realism in both global layout and local geometry. 
    \item We introduce the Welsch Loss that helps pass over high-bias regions in the training data and achieves high-quality residual regression over local geometry artifacts, leading to superior geometry realism in the generated point clouds.
    \item Extensive experiments on SemanticKITTI \citep{Geiger2012CVPR,behley2019iccv}, KITTI360 \citep{Liao2022PAMI}, nuScenes \citep{Caesar_2020_CVPR}, and Waymo Open Dataset \citep{Sun_2020_CVPR} show that the proposed L3DR framework outperforms the state-of-the-art consistently with negligible additional computation cost.
\end{itemize}
\section{Related Work}
\subsection{Diffusion Models}
Diffusion models (DMs) \citep{sohl2015deep,ho2020denoising,song2020denoising,dhariwal2021diffusion,ho2022classifier} comprise a prominent line of contemporary image synthesis methods. Following the early success of Jascha \emph{et al.} \citep{sohl2015deep}, DMs have demonstrated superior scalability \citep{radford2021learning,rombach2022high,ramesh2022hierarchical}, along with various ways to control the image content \citep{zhang2023adding,zhao2023uni}. DMs have garnered increasing attention due to several advantages over other image synthesis methods like Generative Adversarial Network (GAN) \citep{goodfellow2014generative} and Variational Autoencoder (VAE) \citep{kingma2022autoencodingvariationalbayes}. Firstly, DMs are far more capable than VAEs due to their multi-step denoising process being able to provide larger effective model capacity with limited parameters. Secondly, thanks to the simple but robust loss function, DMs are not prone to mode collapse which haunts GANs. 
The great success of image diffusion has recently inspired several LiDAR diffusion models \citep{caccia2019deep,zyrianov2022learning}.
Therefore, we choose to continue on the line of diffusion models for LiDAR generation task.

\subsection{LiDAR Generation}
\paragraph{Range view generation.} 
RV-based generation~\citep{caccia2019deep,haghighi2024taming,hu2024rangeldm, nakashima2021learning,nakashima2024lidar,nakashima2023generative,nakashima2024fast,zyrianov2022learning,ran2024towards,matteazzi2025adverse} converts point clouds into depth images where the height represents elevation and width represents azimuth.  Caccia~\emph{et al.} \citep{caccia2019deep} pioneer this line of research to generate RV images with GAN and VAE. DUSty \citep{nakashima2021learning} improves GAN-based generation by estimating disentangled depth and ray-drop probabilities, and the later DUStyV2 \citep{nakashima2023generative} further decomposes spatial resolution with latent codes to enable LiDAR parameter manipulation. LiDARGen \citep{zyrianov2022learning} first exploits diffusion models on RV-based point cloud generation. R2DM \citep{nakashima2024lidar} applies explicit positional encoding on range-reflectance images to improve generation quality. LidarGRIT \citep{haghighi2024taming} improves latent denoising with an image transformer and raydrop estimation. RangeLDM \citep{hu2024rangeldm} introduces height offsets of lasers to correct traditional spherical RV projection. LiDM \citep{ran2024towards} enables image, text, and semantic conditioned point cloud generation with curve-wise compression. OLiDM \citep{yan2025olidm} improves object generation with a two-stage object-scene diffusion cascade. However, these RV methods all rely on depth images instead of actual 3D geometry, leading to inaccurate local geometries in the generated point clouds.


\paragraph{Re-sampled LiDAR from 3D representations.}
This line of LiDAR data generation methods first generates a scene representation, often from multi-modal data with map conditions, and then samples LiDAR point clouds by simulating the laser reflection process. 
For example, LidarDM \citep{zyrianov2024lidardm} leverages explicit scene-level and object-level mesh diffusion, then re-sample novel LiDAR point clouds with ray casting. 
DynamicCity \citep{bian2024dynamiccity} leverages a novel statio-temporal 4D representation called HexPlane, from which they diffuse novel occupancy map videos.
UniScene \citep{li2024uniscene} further recasts point clouds from occupancy videos with prior guided sampling.
While achieving incredible visual authenticity and temporal consistency, these methods are notoriously resource-hungry, rendering them less economical for LiDAR data generation.

\section{Pilot Study}
\label{sec:pilot}

This section presents whether 2D diffusion models are inherently worse at generating smooth object borders than 3D models. We first introduce the preliminaries on diffusion and deduce a theoretical upper bound for RV gradients as generated by 2D models and 3D models, respectively. We validate our claim with actual gradient distributions of 2D- and 3D-rectified RV images.

\subsection{Preliminaries}

\paragraph{Diffusion models.} Diffusion models generate images by reversing a gradual noise-adding process \citep{sohl2015deep}. The core idea is to corrupt a complex image distribution into a Gaussian noise distribution in the forward process, and learn the reverse process to generate new image samples.

\paragraph{Forward process.} Given a data sample \( x_0\in \mathbb{R}^{u\times v}, x_0\sim q(x_0) \) which represents the target distribution, the forward process gradually adds Gaussian noise over \( T \) steps according to a fixed schedule:
\begin{equation}
    q(x_t \mid x_0) = \mathcal{N}(x_t; \sqrt{\bar{\alpha}_t} \, x_0, (1 - \bar{\alpha}_t) \, \mathbf{I}),
\end{equation}
where \( \alpha_t \in (0, 1) \) are predefined noise scales and \( \bar{\alpha}_t = \prod_{s=1}^{t} \alpha_s \). This process defines a sequence \( x_1, x_2, \ldots, x_T \) where \( x_T \) is approximately standard Gaussian.

\paragraph{Reverse sampling process.} The goal is to learn a reverse process \( p_\theta(x_{t-1} \mid x_t) \) that reconstructs clean samples starting from Gaussian noise:
\begin{equation}
    p_\theta(x_{t-1} \mid x_t) = \mathcal{N}(x_{t-1}; \mu_\theta(x_t, t), \Sigma_\theta(x_t, t)).
    \label{eq:reverse_sampling}
\end{equation}
Typically, a model is trained to predict the added noise \( \epsilon \), and then restore the image using Equation \ref{eq:reverse_sampling}.

\subsection{Theoretical Analysis}
We provide Theorem \ref{th:continuous} where a constant bound can be derived for the image gradient in DDIM-sampled images. We then provide a proof sketch, while the full proof is placed in Appendix Sec.~\ref{ap:proof}.

\begin{theorem}
\label{th:continuous}
Given the assumption that diffusion UNets are Lipschitz continuous, the output image $x_0$ generated by DDIM is locally Lipschitz continuous with respect to the input noise $x_T$. Moreover, the spatial gradient of $x_0$ is bounded:
\[
\| \nabla x_0 \| \leq L \quad \text{for some constant } L \in \mathbb{R}
\]
\end{theorem}
\paragraph{Proof sketch.} \textit{Given the assumption that diffusion models are Lipschitz, and the scalar operations in DDIM are also Lipschitz with finite input parameters and scalar operations, their composition as the diffusion process is also Lipschitz with some constant $L$.}

\begin{corollary}
As a result of Lipschitz continuity throughout the DDIM sampling process, the generated image in theory cannot exhibit arbitrarily sharp spatial transitions. Therefore, DDIM outputs are smooth, with softly transitioned object boundaries.
\end{corollary}

\begin{figure}
  \centering
  \includegraphics[scale=0.6]{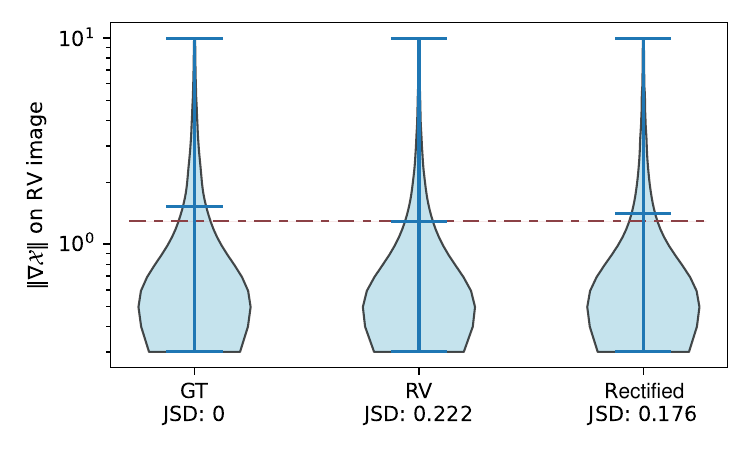}
  \vspace{-0.4cm}
  \caption{\textbf{Empirical validation of Theorem \ref{th:continuous}.} The graph shows the distribution of $\|\nabla x\|$ for GT, vanilla RV diffusion, and our rectified RV, including the corresponding Jensen-Shannon Divergence (JSD) \textit{w.r.t.} the GT.}
  \label{fig:violin}
\end{figure}


\begin{corollary}
\label{cor:3D}
A vital prerequisite for Theorem \ref{th:continuous}, spatial Lipschitz continuity measured in RV image, does not hold when the model itself is a 3D model which accepts back-projected point cloud from an RV image. While 3D models are still generally Lipschitz, the spatial proximity of a point is defined in 3D rather than 2D, adding an additional dimension of limitation. WLOG, we denote the Lipschitz constant for the 3D RRN model as $L_{3D}$ and an RV image rectified with a 3D network as $x_{3d}$. Suppose two adjacent pixels before 3D rectification have a depth difference of $\Delta d$, then we have $\|\nabla x_{3d}\| \leq L_{3D} \times \Delta d$ which is unbounded. We interpret that if $\Delta d$ is large enough, then these two adjacent pixels are out of the effective receptive field (ERF) in 3D operations such as Sparse Convolution or Local Attention, and therefore can generate arbitrarily sharp image borders.
\end{corollary}

\begin{figure*}
  \centering
  \includegraphics[width=0.75\linewidth]{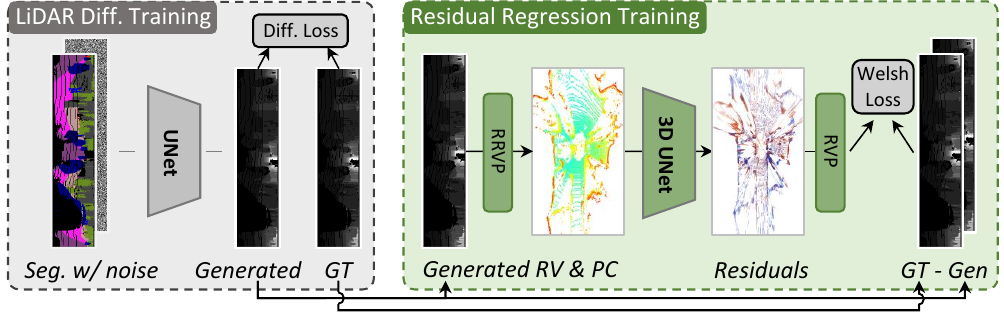}
  \vspace{-0.2cm}
  \caption{\textbf{The training pipeline} of the proposed L3DR framework.
  In the LiDAR diffusion training stage, generated and ground-truth point cloud pairs are collected using semantic-conditioned LiDAR diffusion. In the residual regression training stage, such data pairs are employed to train a 3D network to remove RV artifacts present in the residuals to improve generation quality.}
  \label{fig:architecture}
  \vspace{-0.3cm}
\end{figure*}

\subsection{Empirical Analysis}

To validate that 3D models generate sharper borders than image diffusion models, we plot the empirical distribution of image gradient length $\|\nabla x\|$ for the GT point clouds, vanilla LiDM RV outputs, and our rectified LiDM outputs in Figure \ref{fig:violin}. We remove the dominant $\|\nabla x\| \leq 0.3m$ on planar regions and rare $\|\nabla x\| \geq 10m$ which exceed network ERFs, so that we can compare the remaining geometry-related gradients. According to the figure, range view generated images have lower $\|\nabla x\|$ compared to the ground truth due to the inherent smoothness of diffusion-generated images. However, we are able to restore such sharp borders with the rectification network, which increases overall average image sharpness and reduces JSD from 0.222 to 0.176. We conclude that 3D models generate sharper object borders than 2D models, thus being preferable for rectification of 2D image diffusion results.
\section{Method}
\label{sec:method}

Figure~\ref{fig:architecture} shows the L3DR framework which involves two training stages. In the LiDAR diffusion training, we leverage the RV representation to train a LiDAR diffusion model with conditional semantic input. The diffusion model generates RV point clouds and the corresponding ground-truth. In the residual regression training, we apply such data pairs to train a 3D residual regression network (RRN) that rectifies RV artifacts under the supervision of Welsch Loss.

\begin{figure*}
  \centering
  \includegraphics[width=0.9\linewidth]{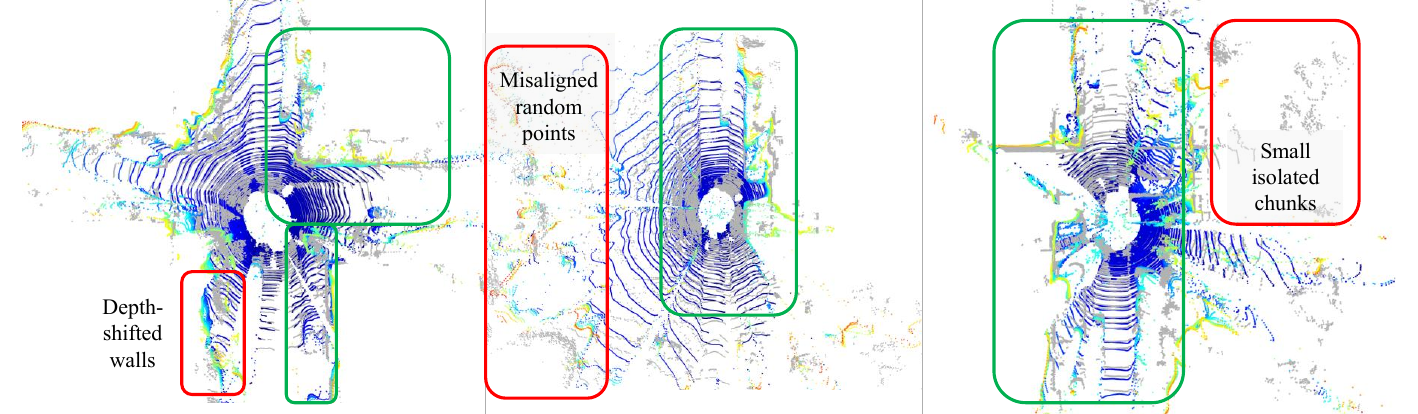}
  \vspace{-0.2cm}
  \caption{\textbf{Visualization of two types of errors in RRN training data.} While the generated point clouds (\textcolor{orange!80!black}{c}\textcolor{yellow!85!black}{o}\textcolor{green!80!black}{l}\textcolor{cyan!90!black}{or}\textcolor{blue}{ed}) approximate the GT (\textcolor{gray}{gray}) in most of the regions with \textbf{high-variance errors}, i.e., RV artifacts as highlighted with green dotted lines, there are also regions with \textbf{high-bias errors} which impede training, including (1) shifted walls, (2) random points on the leaves where laser hits are hard to predict, and (3) isolated chunks with consistent depth error. These bias-dominated regions are harmful for RRN training.}
  \label{fig:labels}
\end{figure*}

\begin{figure}
  \centering
  \includegraphics[width=\linewidth]{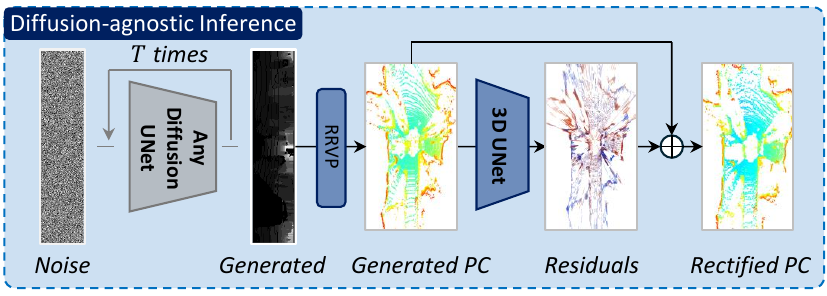}
  \caption{\textbf{The inference pipeline} of the proposed L3DR.
  }
  \label{fig:inference}
\end{figure}

\subsection{Range View Projection}
\label{sec:rvp}
Range view projection (RVP) and reverse range view projection (RRVP) are the key components that connects RV images with point clouds. Given a point cloud $ P \in \mathbb R^{N\times 3}$ consisting of $N$ points in Euclidean coordinates and a range-view depth image $x\in \mathbb{R}^{u\times v}$, the range view projection converts a point $(p_x,p_y,p_z)^\intercal \in P$ into a pixel $(u_i, v_i)^\intercal$ with depth value $d_i$:

\begin{align}
    &\begin{aligned}
    \begin{pmatrix} u_i\\
    v_i\\
    \end{pmatrix}&=
    \begin{pmatrix}
        \sigma_u \arctan\left(p_z \Big/ \sqrt{(p_x^2+p_y^2)}\right),\\
        -\sigma_v \arctan(p_x / p_y),\\
    \end{pmatrix}\\
    d_i&=\sqrt{(p_x^2 + p_y^2 + p_z^2)}\\
    \end{aligned}
\end{align}
Where $\sigma_u, \sigma_v$ are pixels per radian in elevation and azimuth directions, respectively. They are usually anisotropic, \emph{i.e.}, the vertical resolution is sparse while horizontal resolution is dense. For LiDARs with non-equally angled lasers, $\sigma_u$ can be a function of elevation.
Similarly, the reverse range view projection converts a pixel $(u_i, v_i)$ with depth $d_i$ into a point $(p_x,p_y,p_z)^\intercal$: 
\begin{align}
    \begin{pmatrix}
                p_x \\
                p_y \\
                p_z
                \end{pmatrix}=
    \begin{pmatrix}
    d_i \cos(u_i / \sigma_u) \cos(v_i / \sigma_v),\\
    -d_i \cos(u_i / \sigma_u) \sin(v_i / \sigma_v),\\
    d_i \sin(u_i / \sigma_u).
    \end{pmatrix},
\end{align}

Note that RRVP is lossless but RVP is lossful, because multiple points may be projected to the same pixel. Empirically, RVP is lossless when the point cloud projection structure is preserved and RV image size is equal to or larger than the laser angular resolution.

\subsection{LiDAR Diffusion Training}
In order to generate ground-truth and diffusion-generated point cloud pairs for the following training stage, we re-train a state-of-the-art conditional LiDAR diffusion model, LiDM \citep{ran2024towards}, on KITTI, nuScenes, and WOD to generate LiDAR point clouds based on segmentation map as a condition. Specifically, the RV image is first compressed with a VQ-VAE \citep{van2017neural}, then a classical diffusion UNet is leveraged to predict Gaussian noise added in the compressed latent space. To enable conditional generation, a down-sampled segmentation color map is appended to latent space as control input. 
We collect the ground-truth RV $x_{gt}$ and corresponding generated $x_{gen}$ after convergence to provide training data for the next stage.

We highlight the importance of LiDM due to its capability to generate point clouds that match closely with ground truths, and yet remain imbued with RV artifacts on a limited scale, as depicted in Figure \ref{fig:labels}.
However, we also highlight that our framework is general and not restricted to LiDM, given that an alternative LiDAR diffusion method can generate such closely approximated point cloud pairs.

\subsection{Residual Regression Training}
\paragraph{Pipeline.} After obtaining a generated point cloud, we feed it into a 3D backbone named Residual Regression Network (RRN) to accurately regress and cancel out the hidden RV artifacts to achieve better geometry-realism. Specifically, we first obtain the generated point cloud $P_{gen}=\mathrm{RRVP}(x_{gen})$ through reverse range view projection. Then, $P_{gen}$ is fed into a 3D backbone $\mathbf F: \mathbb R^{N\times k} \rightarrow \mathbb R^{N\times 3}$ to obtain 3D offsets $O=\mathbf F(P_{gen})$. For a normal unconditional 3D RRN, $k=3$ which represents the input point coordinates; otherwise, $k=6$ which accepts both the coordinates and a semantic color map. Finally, we project the output offsets onto the radial directions of the original point cloud to obtain the final residuals $\hat O=P_{gen}\mathrm{diag}(P_{gen}O^\intercal) \big/ \sqrt{\mathrm{diag}(P_{gen}P_{gen}^\intercal)}$.

\paragraph{Decomposition of the learning objective.} From a bias–variance decomposition viewpoint, bias measures the consistent deviation between the diffusion-generated depth and the ground truth, while variance captures variable RV artifacts, as depicted in Figure \ref{fig:labels}. High bias errors naturally emerge in large chunks sharing the same segmentation class, \textit{i.e.}, insufficient semantic constraints introduce unrealistic yet coherent hallucinations such as a flat wall being interpreted as slanted or a tree as distant rather than nearby. These errors are particularly harmful to geometry correction since the residuals stem from biased interpretation of the scene rather than stochastic RV artifact error. More discussion is given in section \ref{sec:regress_all}.
\label{sec:decomposition_of_objective}

Although high-bias and high-variance errors often coexist, RV artifacts possess much smaller error magnitude compared to high-bias errors. This distinction enables a soft separation of bias- and variance-dominated regions based on residual magnitude, which inspires the introduction of the Welsch loss to suppress large, bias-driven deviations.

\paragraph{Loss function.} After obtaining the model output and GT, we propose Welsch Loss to remove the effect of erratic high-bias areas in training data to focus on rectifying high-variance RV artifacts.
Specifically, Welsch's Function is defined as:
\begin{equation}
    \psi_\nu(x)=1-\exp\left(-\frac{x^2}{(2\nu ^2)}\right)
\end{equation}
Which is an upside-down bell curve and $\nu>0$ is a width parameter. We then define the RRN loss function as:
\begin{equation}
    L_{RRN}=\mathrm{mean}\Big(\psi_\nu \big(\mathrm{RVP}(P_{gen}+\hat O) - x_{gt}\big)\Big)
\end{equation}
Where $\mathrm{RVP}(\cdot)$ denotes range view projection, and $\mathrm{mean}(\cdot)$ is the average operation. We choose the suitable $\nu$ for all datasets and fix them for other experiments.

\begin{table*}[t]
  \centering
  \resizebox{0.75\linewidth}{!}{
  \begin{tabular}{cccccc}
    \toprule
&& \multicolumn{2}{c}{\textbf{Perceptual}} & \multicolumn{2}{c}{\textbf{Distributional}}\\
\cmidrule(lr){3-4} \cmidrule(lr){5-6}
Task&Method&FSVD$\downarrow$	&FPVD$\downarrow$	&JSD$\downarrow$	&MMD$\times 10^{(-4)}\downarrow$ \\
\midrule
\multirow{10}{*}{\specialcell{KITTI360\\Unconditional}} &LiDARGAN	    \citep{caccia2019deep} &183.4	&168.1	&0.272	&4.74\\
&LiDARVAE	    \citep{caccia2019deep}	&129.9	&105.8	&0.237	&7.07\\
&ProjectedGAN	\citep{sauer2021projected}	&44.7	&33.4	&0.188	&\textbf{2.88}\\
&UltraLiDAR	    \citep{xiong2023learning}	&72.1	&66.6	&0.747	&17.12\\
\cmidrule{2-6}
&LiDARGen(1160s)	\citep{zyrianov2022learning}	&39.2	&33.4	&0.188	&\textbf{2.88}\\
&LiDARGen(50s)	\citep{zyrianov2022learning} &480.6	&400.7	&0.506	&9.91\\
&LDM(50s)	    \citep{rombach2022high}	&70.7	&61.9	&0.236	&5.06\\
&\cellcolor{lightgray!30}R2DM(50s)	    \citep{nakashima2024lidar}	&\cellcolor{lightgray!30} 36.8	&\cellcolor{lightgray!30} 30.9	&\cellcolor{lightgray!30} \underline{0.168}  &\cellcolor{lightgray!30} \underline{2.92}\\
&\cellcolor{lightgray!30}LiDM(50s)	    \citep{ran2024towards}	&\cellcolor{lightgray!30}38.8	&\cellcolor{lightgray!30}29.0	&\cellcolor{lightgray!30}0.211	&\cellcolor{lightgray!30}3.84\\
&\cellcolor{lightgray!50} Ours-R2DM	   & \cellcolor{lightgray!50}\underline{35.9} & \cellcolor{lightgray!50}\underline{28.2} & \cellcolor{lightgray!50}\textbf{0.165} & \cellcolor{lightgray!50}2.90\\
&\cellcolor{lightgray!50} \textcolor{cyan!90!black}{$\Delta$ \textit{Increment w.r.t. R2DM}} & \cellcolor{lightgray!50}\textcolor{cyan!90!black}{+2.4\%} & \cellcolor{lightgray!50}\textcolor{cyan!90!black}{+8.7\%} & \cellcolor{lightgray!50}\textcolor{cyan!90!black}{+1.8\%} & \cellcolor{lightgray!50}\textcolor{cyan!90!black}{+0.7\%}\\

&\cellcolor{lightgray!50} Ours-LiDM	        	&\cellcolor{lightgray!50}\textbf{35.8}	&\cellcolor{lightgray!50}\textbf{26.1}	&\cellcolor{lightgray!50}0.182 &\cellcolor{lightgray!50}3.27\\
&\cellcolor{lightgray!50} \textcolor{cyan!90!black}{$\Delta$ \textit{Increment w.r.t. LiDM}} &\cellcolor{lightgray!50} \textcolor{cyan!90!black}{+7.7\%}	&\cellcolor{lightgray!50}\textcolor{cyan!90!black}{+10.0\%}	&\cellcolor{lightgray!50}\textcolor{cyan!90!black}{+13.7\%}	&\cellcolor{lightgray!50}\textcolor{cyan!90!black}{+14.8\%}\\
    \midrule
    \midrule
    \multirow{3}{*}{\specialcell{nuScenes\\Conditional}}    &\cellcolor{lightgray!30}LiDM \citep{ran2024towards}	&\cellcolor{lightgray!30}86.6	&\cellcolor{lightgray!30}74.8	&\cellcolor{lightgray!30}0.145	&\cellcolor{lightgray!30}2.81\\
    & \cellcolor{lightgray!50}Ours-LiDM-Sem	 &\cellcolor{lightgray!50}\textbf{81.3}	&\cellcolor{lightgray!50}\textbf{67.0}	&\cellcolor{lightgray!50}\textbf{0.133}	&\cellcolor{lightgray!50}\textbf{2.72}\\
    & \cellcolor{lightgray!50}\textcolor{cyan!90!black}{$\Delta$ \textit{Increment w.r.t. LiDM}}  &\cellcolor{lightgray!50}\textcolor{cyan!90!black}{+6.12\%}	&\cellcolor{lightgray!50}\textcolor{cyan!90!black}{+10.43\%}	&\cellcolor{lightgray!50}\textcolor{cyan!90!black}{+8.28\%}	&\cellcolor{lightgray!50}\textcolor{cyan!90!black}{+3.20\%}\\

    \midrule
    \multirow{3}{*}{\specialcell{Waymo\\Conditional}}       &\cellcolor{lightgray!30}LiDM \citep{ran2024towards}	&\cellcolor{lightgray!30}21.4	&\cellcolor{lightgray!30}21.9	&\cellcolor{lightgray!30}0.104	&\cellcolor{lightgray!30}1.30\\
    & \cellcolor{lightgray!50}Ours-LiDM-Sem	&\cellcolor{lightgray!50}\textbf{18.3}	&\cellcolor{lightgray!50}\textbf{20.3}	&\cellcolor{lightgray!50}\textbf{0.086}	& \cellcolor{lightgray!50}\textbf{1.25}\\
    &\cellcolor{lightgray!50}\textcolor{cyan!90!black}{$\Delta$ \textit{Increment w.r.t. LiDM}} &\cellcolor{lightgray!50}\textcolor{cyan!90!black}{+14.49\%}	&\cellcolor{lightgray!50}\textcolor{cyan!90!black}{+7.31\%}	&\cellcolor{lightgray!50}\textcolor{cyan!90!black}{+17.31\%}	&\cellcolor{lightgray!50}\textcolor{cyan!90!black}{+3.85\%}\\

    \bottomrule
  \end{tabular}

  }
  \caption{Benchmarking of unconditional generation on KITTI360 and semantic-conditioned generation on nuScenes and Waymo.
  For the semantic-conditioned experiments, RRN takes segmentation map as additional input for optimal performance. Gray areas highlight direct comparisons with the baselines.}
  \label{tab:main}
\end{table*}

\subsection{Diffusion-agnostic Inference}
During inference, because the RRN has been well trained to remove RV artifacts, the diffusion model can be replaced with arbitrary LiDAR diffusion model, as depicted in Figure \ref{fig:inference}. Empirically, the RRN can generalize to arbitrary unseen unconditional diffusion networks even if the RRN is only trained on segmentation-conditioned data generated by LiDM. Specifically, during inference, we generate novel $x_{gen}'$ with arbitrary LiDAR diffusion model, project RV into a point cloud $P_{gen}'=\mathrm{RRVP}(x_{gen}')$, calculate the point wise residual $\hat O'=P_{gen}'diag(P_{gen}'\mathbf F(P_{gen}')^\intercal)\big/ \sqrt{\mathrm{diag}(P_{gen}'P_{gen}'^\intercal)}$, and obtain the final rectified point cloud $P_{ref}'=P_{gen}'+\hat O'$.
\section{Experiments}

\subsection{Main Results}

In this section, we first introduce our experiment setup, then report the generation metrics in Table \ref{tab:main}-\ref{tab:kitti_cond} and provide visualization in Figure \ref{fig:semantic_conditioned}-\ref{fig:unconditioned}. Module ablation and time analysis are listed in Table \ref{tab:ablation}-\ref{tab:time}. Ablation on other datasets and a parameter tuning experiment can be found in Appendix Sec.~\ref{ap:exp}.

\subsection{Experiment Setup}
\label{sec:experiment_setup}
\paragraph{Datasets.} We train and evaluate our method on SemanticKITTI \citep{Geiger2012CVPR,behley2019iccv}, KITTI360 \citep{Liao2022PAMI}, nuScenes \citep{Caesar_2020_CVPR}, and Waymo Open Dataset \citep{Sun_2020_CVPR}. 
All datasets are split into train-validation-test according to official recommendations. We generate RRN training data by conditional LiDM inference on training set conditions, \emph{i.e.}, segmentation maps, and report the metrics of the last-epoch model on validation set.

\begin{table*}[t]
  \centering
  \resizebox{0.75\linewidth}{!}{
  \begin{tabular}{ccccc|cccc}
    \toprule
& \multicolumn{4}{c}{\textbf{Semantic-Map-Conditioned}} & \multicolumn{4}{c}{\textbf{Front-Image-Conditioned}}\\
\cmidrule(lr){2-5} \cmidrule(lr){6-9}
Method&FSVD$\downarrow$	&FPVD$\downarrow$	&JSD$\downarrow$	&MMD$\times 10^{(-4)}\downarrow$ &FSVD$\downarrow$	&FPVD$\downarrow$	&JSD$\downarrow$	&MMD$\times 10^{(-4)}\downarrow$\\
\midrule
LiDARGen	&31.7	&30.1	&0.130	&5.18   & -	    & -	    & -     & -  \\
LDM     	&21.3	&20.3	&0.088	&3.73	&35.9	&26.5	&0.26	&3.80\\
\rowcolor{lightgray!30} LiDM	    &\underline{20.2}	&17.7	&\underline{0.072}	&3.16	&\textbf{32.5}	&\underline{25.8}	&\underline{0.21}	&\underline{3.69}\\

\rowcolor{lightgray!50} Ours   &20.4	&\underline{15.0}	&\textbf{0.070}	&\underline{1.69}	&\underline{33.6}	&\textbf{24.9}	&\textbf{0.18}	&\textbf{3.30}\\
\rowcolor{lightgray!50} \textcolor{cyan!90!black}{$\Delta$ \textit{Inc.}} & \textcolor{cyan!90!black}{-1.0\%}	&\textcolor{cyan!90!black}{+15.3\%}	&\textcolor{cyan!90!black}{+2.8\%}	&\textcolor{cyan!90!black}{+46.5\%} &\textcolor{cyan!90!black}{-3.4\%}	&\textcolor{cyan!90!black}{+3.5\%}	&\textcolor{cyan!90!black}{+12.2\%}	&\textcolor{cyan!90!black}{+10.6\%}\\

\rowcolor{lightgray!50} Ours-Sem   &\textbf{15.8}	&\textbf{12.9}	&\textbf{0.070}	&\textbf{1.50}  & -	    & -	    & -     & - \\
\rowcolor{lightgray!50} \textcolor{cyan!90!black}{$\Delta$ \textit{Inc.}} 
&\textcolor{cyan!90!black}{+21.8\%}	&\textcolor{cyan!90!black}{+27.1\%}	&\textcolor{cyan!90!black}{+2.8\%}	&\textcolor{cyan!90!black}{+52.5\%}
  & -	    & -	    & -     & - \\
    \bottomrule
  \end{tabular}
  }
  \caption{\textbf{Comparison of conditional LiDAR point cloud generation} on SemanticKITTI and KITTI360. Gray areas highlight direct comparisons with the baseline, LiDM. `Ours-Sem' denotes our method with segmentation input to the RRN.}
  \label{tab:kitti_cond}
\end{table*}

\begin{figure*}
  \centering
  \includegraphics[width=0.9\linewidth]{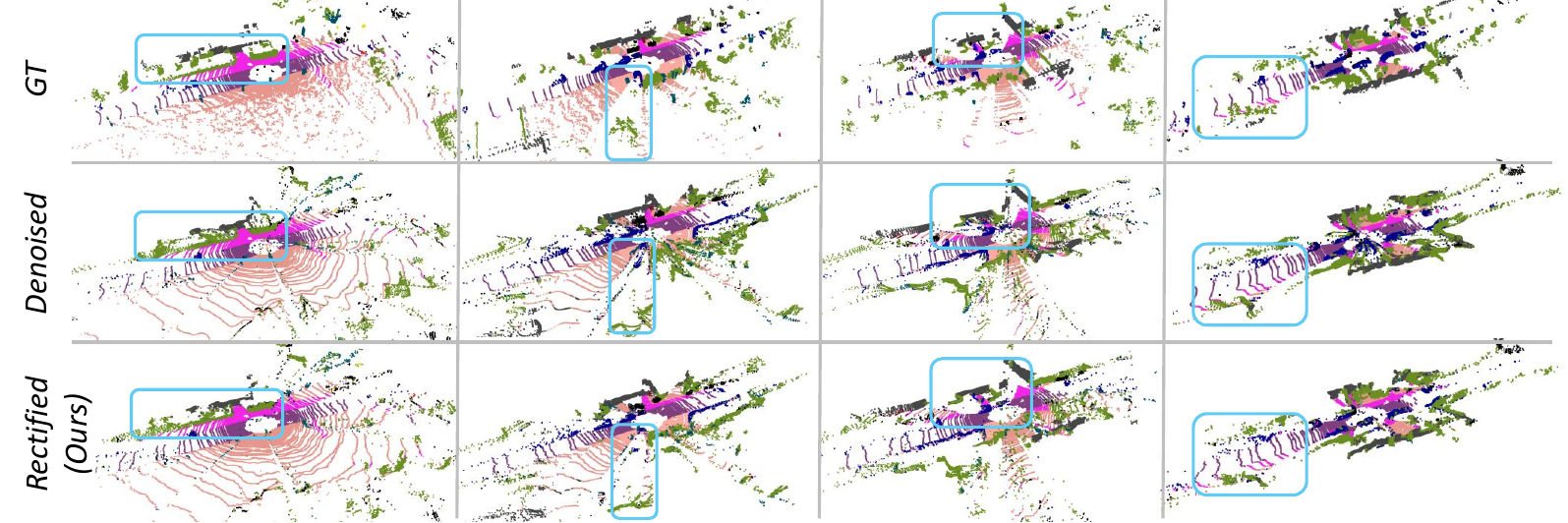}
  \vspace{-0.2cm}
  \caption{\textbf{Visualization of conditional generation on SemanticKITTI.} Cyan regions highlight the improved RV artifacts from the diffusion-generated (i.e., denoised) data to our rectified data, such as depth bleeding and wavy surfaces.}
  \label{fig:semantic_conditioned}
\end{figure*}

\paragraph{Metrics.} We report both perceptual metrics including the Fr\'echet Sparse Volume Distance (FSVD) and Fre\'chet Point Voxel Distance (FPVD) proposed by \cite{ran2024towards}, and conventional distributional metrics, Jensen-Shannon Divergence (JSD) and Minimum Matching Distance (MMD).

\begin{table*}[t]
  \centering
  \resizebox{0.7\linewidth}{!}{
  \begin{tabular}{lcccccc}
    \toprule
\multicolumn{3}{c}{\textbf{Configurations}} & \multicolumn{2}{c}{\textbf{Perceptual}} & \multicolumn{2}{c}{\textbf{Statistical}}\\
\cmidrule(lr){1-3} \cmidrule(lr){4-5} \cmidrule(lr){6-7}
Backbone & Loss & RRN Sem. Input&FSVD$\downarrow$	&FPVD$\downarrow$	&JSD$\times10^{(-2)}\downarrow$	&MMD$\times 10^{(-5)}\downarrow$\\
\midrule
/ (baseline) & / &- &18.3	&15.3	&7.1	&16.2\\
\midrule
SPUNet & Welsch & -  &\underline{16.4}	&12.1	&\textbf{6.7}	&16.7\\
SPUNet & MSE & -   &26.3	&25.1	&7.0	&\textbf{12.6}\\
SPUNet & Welsch & \checkmark &\textbf{12.5}	&\textbf{10.7}	&\textbf{6.7}	&\underline{15.0}\\
\midrule
PTV3 & Welsch &	 -   &17.7	&13.9	&\underline{6.8}	&17.2\\
PTV3 & MSE & -   &42.4	&42.6	&7.4	&18.8\\
PTV3 & Welsch & \checkmark   &\textbf{12.5}	&\underline{10.8}	&\textbf{6.7}	&15.8\\
\midrule
2D UNet & Welsch&  -  &19.2	&16.4	&7.1	&16.3\\
    \bottomrule
  \end{tabular}
  }
  \vspace{-0.2cm}
  \caption{\textbf{Ablation experiment} on SemanticKITTI, including RRN backbone structure, loss function, semantic-map input to RRN, and a fair baseline using a 2D image Unet instead of a 3D UNet.}
  \vspace{-0.4cm}
  \label{tab:ablation}
\end{table*}

\begin{table}[t]
  \centering
  \resizebox{\linewidth}{!}{
  \begin{tabular}{lcccc}
    \toprule
    Method & Network & \# Params (M) & Steps & Time (ms) \\
    \midrule
    R2DM & Eff-UNet & 31.10 &50 & 579.83 \\
    LiDM & VQ-VAE, UNet & 257.77 &50 & 557.36 \\
    +Ours & SPUNET & +37.90 & +1 & +19.65 \\
    \bottomrule
  \end{tabular}
  }
  \caption{\textbf{Computational overhead} on KITTI360. our method introduce very slight computational overhead over the baselines.}
  \label{tab:time}
  \vspace{-0.5cm}
\end{table}

\paragraph{Training.} Our diffusion model processes depth values of size without logarithmic scaling. For 64-laser KITTI and Waymo, we naturally designate the RV image size as $(64,1024)$. 
However, we discover that adopting a $(32,1024)$ image size on the 32-laser nuScenes leads to network divergence, and that LiDM \citep{ran2024towards} did not open-source its nuScenes config. We solve this problem by over-provisioning the RV image size to $(64, 1024)$ for nuScenes, which stabilizes training without affecting generation quality as discussed in Section \ref{sec:rvp}. We hypothesize that a large-enough image size is vital to training dynamics. All networks are trained with $4\times$ RTX 4090 24G up to 150 epochs.

\paragraph{KITTI unconditional generation.} We compare our L3DR method with existing LIDAR generation methods on KITTI unconditional generation task in upper half of Table \ref{tab:main}. Unconditional R2DM and LiDM are adopted as the first-stage diffusion model for L3DR. With the powerful capability of removing RV artifacts, L3DR sets a new state-of-the-art, achieving $35.8(\Delta7.7\%)$ FSVD, $26.1(\Delta10.0\%)$ FPVD, and $0.182(\Delta13.7\%)$ JSD, greatly surpassing the baseline method LiDM. Similar results are also observed when a non-latent diffusion method R2DM is applied, which signals the general applicability of our L3DR framework, achieving the best JSD of $0.165(\delta 1.8\%)$. While L3DR does not top the MMD metric, our method still provides a average $7.3\%$ improvement, and is comparable to the best-performing ProjectedGAN which scores $2.88$. We conclude that L3DR achieves state-of-the-art global photo- and geometry-realism on unconditional LiDAR generation task.

\paragraph{KITTI conditional generation.} We compare L3DR with conditional LIDAR generation methods on both SemanticKITTI and KITTI360, respectively in table \ref{tab:kitti_cond}. L3DR again achieved better performance compared to the baselines, gaining $15.0(\Delta15.3\%)$ FPVD, $0.07(\Delta2.8\%)$ JSD, and an astonishing $1.69\times 10^{-4}(\Delta46.5\%)$ MMD on SemanticKITTI, and $24.9(\Delta3.5\%)$ FPVD, $0.18(\Delta12.2\%)$ JSD, and $3.3\times 10^{-4}(\Delta10.6\%)$ MMD on KITTI360. Meanwhile, L3DR with semantic-map input improves generation quality consistently, providing further $10.2\%$ average performance boost on all metrics upon raw L3DR. We conclude that L3DR significantly improves conditional generation capability compared to the baselines.

\paragraph{Performance on other datasets.} We examine the semantic-conditioned generation metrics of our L3DR method on nuScenes and Waymo, in the lower half of Table \ref{tab:main}. L3DR exhibits consistent improvements on all datasets, improving all metrics by an average of $11.6\%$ and $7.0\%$ on nuScenes and Waymo conditional generation, respectively. We conclude that L3DR is widely applicable and sets new state-of-the-arts on the nuScenes and Waymo semantic-conditioned generation tasks.

\begin{figure*}[h!]
  \centering
  \includegraphics[width=0.76\linewidth]{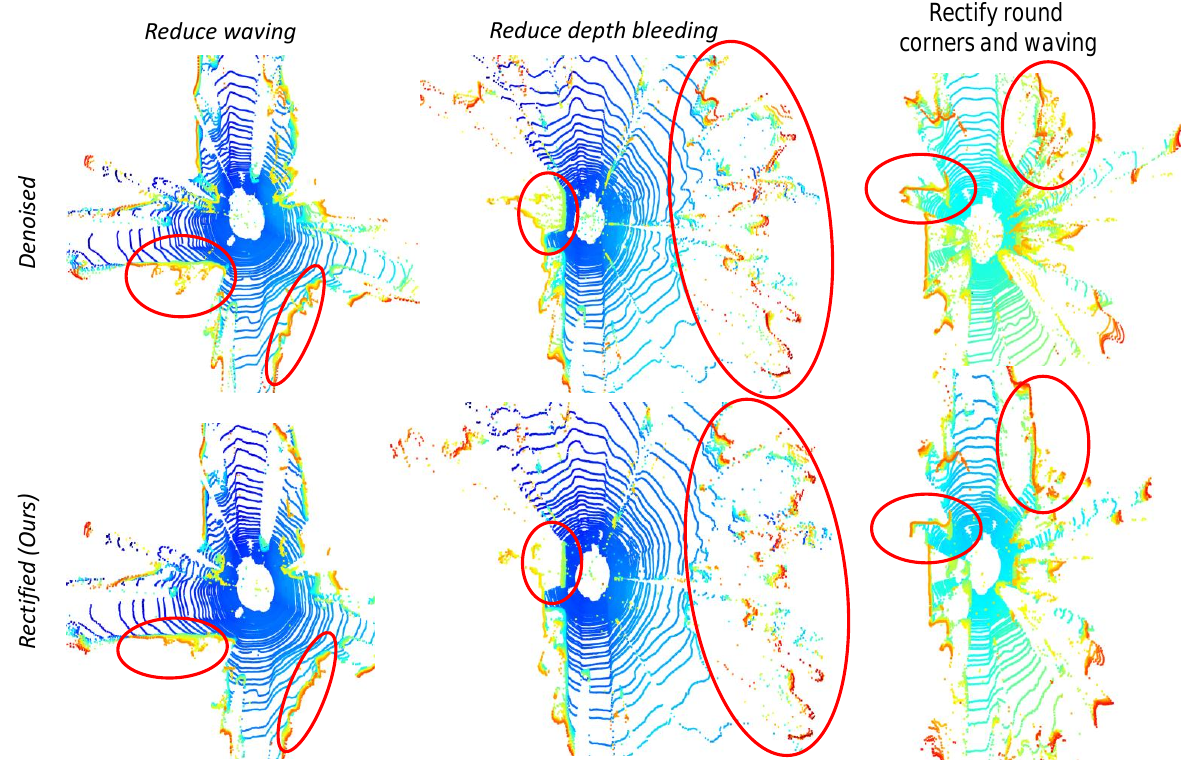}
    \vspace{-0.3cm}
  \caption{\textbf{Visualization of unconditional generation on KITTI360.} Red regions highlight the the improvements from the diffusion-generated (i.e., denoised) data to our rectified data.}
  \label{fig:unconditioned}
  \vspace{-0.2cm}
\end{figure*}

\subsection{Other Results}

\paragraph{Ablation.} We ablate various components of the L3DR framework, including the backbone architecture between SPUNET \citep{spconv2022} and PTV3 \citep{wu2024point}, loss choice, the usage of semantic-map input in RRN, and a 2D RRN instead of 3D RRN. We calculated metrics with voxelized GT instead of original GT for fast evaluation, which results in slightly different figures. As listed in Table \ref{tab:ablation}, applying MSE Loss generally deteriorates performance compared to baseline, with FSVD and FPVD almost doubling from $18.3, 15.3$ to $26.3, 25.1$ for SPUNET and $42.4, 42.6$ for PTV3. The best overall performance attributes to SPUNET with Welsch loss and semantic-map input. Additionally, substituting 3D UNet with a 2D Image Unet brings considerable degradation on all metrics. We default to SPUNET and Welsch Loss for all other experiments.

\paragraph{Time analysis.} 
We provide a simple time comparison to demonstrate the efficiency of L3DR framework. On an RTX 4090 GPU 24G, the additional time for RRN rectification is $19.65$ ms which is negligible compared to LiDM and R2DM sampling processes which both require more than $550$ ms. The 37.90 M additional parameters from RRN are also lightweight, especially when compared to LiDM. We conclude that the L3DR framework can greatly improve generation quality with negligible additional parameter and inference cost.

\vspace{-0.2cm}
\section{Conclusion}
\vspace{-0.2cm}
We have proposed L3DR, a 3D-aware LIDAR Diffusion and Rectification framework which combines 2D RV LiDAR diffusion with a 3D residual regression network to achieve both layout-realism and geometry-realism on conditional and unconditional LiDAR diffusion task. By combining our Welsch Loss with the training data generated from semantic-conditioned diffusion, we are able to train a powerful Residual Regression Network to remove RV artifacts such as depth bleeding and wavy surfaces. Our method exhibits state-of-the-art performance and broad applicability on different LiDAR diffusion methods, while requiring only negligible additional computational cost, making it a perfect choice for cost-efficient generation of high-fidelity novel LiDAR point clouds.

\section*{Acknowledgments}
This study is funded by the Ministry of Education Singapore, under
the Tier-2 project scheme with project number MOET2EP20123-0003.\\
{
    \small
    \bibliographystyle{ieeenat_fullname}
    \bibliography{main}
}

\clearpage
\setcounter{page}{1}
\maketitlesupplementary

\appendix

\section{Technical Appendices and Supplementary Material}

\subsection{Related Work on Non-RV 3D Diffusion}
3D based methods directly diffuse point coordinates on 3D or compressed 3D space instead of RV image space. Specifically, they have seen applications on object point clouds \citep{luo2021diffusion,melas2023pc2,tyszkiewicz2023gecco,zhou20213d,yin2025shapegpt}, mesh representation \citep{gupta20233dgen,liu2023meshdiffusion,zhuravlev2025denoising}, and neural implicit fields \citep{cheng2023sdfusion,erkocc2023hyperdiffusion,vahdat2022lion}.
However, due to the natural preference for local geometry rather than global structure, 3D diffusion has seen twists when applied to LiDAR point clouds.
UltraLiDAR \citep{xiong2023learning} circumvents the issue by encoding point clouds into a VQ-VAE feature map in BEV view. Such representation empowers accurate local geometry but disregards global self-occlusion, resulting in inaccurate projection relationships.
LiDiff \citep{nunes2024scaling} approaches point cloud completion by 3D diffusion on the point level, but fails on direct generation. While these methods share similarities with our residual regression network, we highlight that they are not designed to recover local geometric properties including the LiDAR projection structure, therefore being unsuitable for the 3D-aware LiDAR generation task.

\subsection{Full Proof of Theorem \ref{th:continuous}}
\label{ap:proof}
Without loss of generality (WLOG), we choose DDIM \citep{song2020denoising} as the analyze target, as it is of general form and has been widely applied. We then continue to derive a constant bound for the image gradient in DDIM-sampled images, and show that such bound for a 3D network is much looser.

\begin{assumption}
Let $\epsilon_\theta: \mathbb{R}^{u\times v} \times \mathbb{R} \rightarrow \mathbb{R}^n$ denote the noise prediction network used in DDIM, where $\epsilon_\theta(x_t, t)$ predicts noise at time step $t$. We assume $\epsilon_\theta$ is spatially Lipschitz continuous: 
\begin{equation}
    \label{eq:lipschitz}
    \begin{split}
    \|\nabla \epsilon_\theta(x, t) \| \leq L_{\theta,t} \| \nabla x \|,\quad \forall x \in \mathbb{R}^{u\times v}.
    \end{split}
\end{equation}
Where $L_{\theta,t}$ is a finite constant. In practice, the network $\epsilon_\theta$ is often implemented using convolutions, attentions and element-wise operations, which lead to such spatial smoothness.
\end{assumption}
\noindent
According to \citet{song2020denoising}, DDIM update rule is given by:
\begin{equation}
    \label{eq:ddim}
    x_{t-1} = \sqrt{\alpha_{t-1}} \cdot \hat{x}_0 + \sqrt{1 - \alpha_{t-1}} \cdot \epsilon_\theta(x_t, t).
\end{equation}
\begin{equation}
    \label{eq:ddim-2}
    \hat{x}_0 = \frac{x_t - \sqrt{1 - \alpha_t} \cdot \epsilon_\theta(x_t, t)}{\sqrt{\alpha_t}}.
\end{equation}
Where $\alpha_t$ are constants determined by a diffusing schedule.
Each sampling step can be shorthanded as a function $x_{t-1}=f_t(x_t)$, and the final output is:
\begin{equation}
\label{eq:sampling}
x_0 = f_1 \circ f_2 \circ \cdots \circ f_T(x_T),
\end{equation}
where $x_T \sim \mathcal{N}(0, I)$ is the initial noise.

\paragraph{Theorem 1.}
\textit{Under the above assumptions, the output image $x_0$ generated by DDIM is locally Lipschitz continuous with respect to the input noise $x_T$. Moreover, the spatial gradient of $x_0$ is bounded:}
\[
\| \nabla x_0 \| \leq L \quad \text{for some constant } L \in \mathbb{R}
\]

\begin{proof}
Each function $f_t$ is spatially Lipschitz continuous due to the spatial Lipschitz property of $\epsilon_\theta$. Scalar operations in Eq.\ref{eq:ddim}-\ref{eq:ddim-2} is also Lipschitz with constant $L_{\epsilon,t}$ due to the finite input parameters and scalar operations. Then their composition in Eq.~\ref{eq:sampling} is also spatially Lipschitz with constant: $L = \prod_{t=1}^{T} L_{\theta,t}L_{\epsilon,t}$.


\end{proof}

\subsection{Qualitative intuition for Theorem \ref{th:continuous}.} To qualitatively visualize the differences in noise tolerance between 2D range view images and 3D point clouds, We add exactly the same Gaussian noises to the same point cloud, which is displayed in both RV and 3D in Figure \ref{fig:2d_3d_intuition}. Adding large noises ($\sigma=5$m) does not change the perceivable semantics in images, but dramatically shuffles the point cloud to a cloud of mess, making it beyond human understanding. To preserve the spatial structure after adding noise, the noise has to be very tiny ($\sigma=0.2$m). While the noise is still clearly visible on 3D according to the wiggling points, such small noise is unpercievable in 2D. Therefore, we conclude that 2D is much more tolerant and robust to large noises but ignores small noises, while 3D is more sensitive and intolerable to large noises but is better at processing small noises. Therefore, 2D RV is better at generating layouts from scratch using diffusion models, while 3D networks are better at rectifying local geometry of the diffusion-generated point clouds.

\begin{figure*}[h!]
  \centering
  \includegraphics[width=0.9\linewidth]{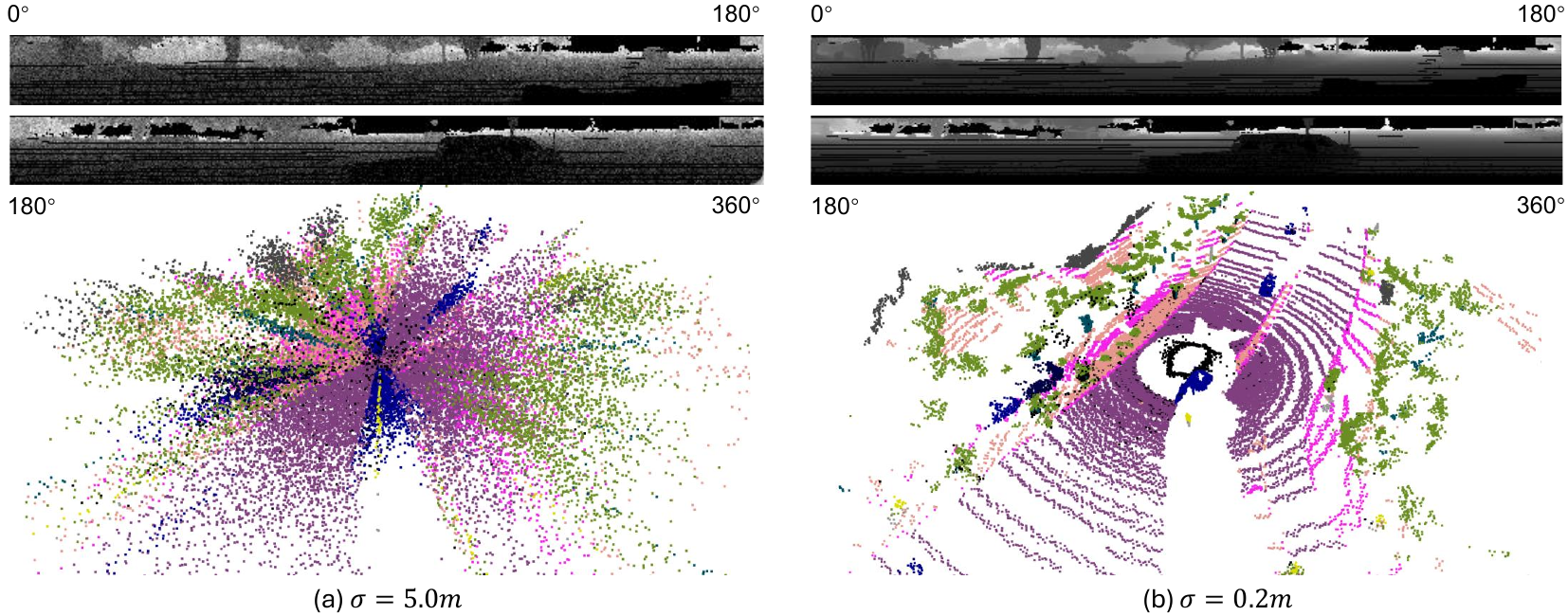}
  \caption{\textbf{Intuitive validation for Theorem \ref{th:continuous}}, \textit{i.e.,} 2D and 3D models have inherently different sensitivity to disturbances. (a) Adding large noises ($\sigma=5$m) does not hinder human understanding of image contents on RV, yet completely sets the point cloud into a mess. (b) Adding small noises ($\sigma=0.2$m) adds visible noises in the point cloud while retaining its spatial structure, but this noise scale is almost non-percievable in RV.}
  \label{fig:2d_3d_intuition}
\end{figure*}

\subsection{Additional Experiment Setups}
\label{ap:exp_setup}
\paragraph{Training details.} For the first-stage diffusion model LiDM, we keep the default set of parameters from the original paper for SemanticKITTI, which is conducted with $(64, 1024)$ RV image size which is shrunk into $(16, 128)$ size for latent diffusion with $16$ channels. Unlike the original paper, we also place $(64, 1024)$ for nuScenes to avoid divergence, and disable the logarithmic scaling of depth values. The LiDM model is trained with AdamW optimizer under $1\times 10^{(-6)}$ learning rate for up to 150 epochs. For the second stage RRN model, we follow the standard Pointcept training protocol without task-specific hyperparameter tuning. All RRN models are trained for a fixed 50 epochs using AdamW with OneCycleLR (max LR 0.002, weight decay 0.005), while certain modules in PTV3 have 0.0002 learning rate according to default configuration. No separate validation set or early stopping is used; instead, we observe stable convergence across datasets and report last-epoch performance. This fixed protocol ensures fair and reproducible comparisons across datasets and settings.

\paragraph{Licenses.} The datasets we use are: 1) KITTI which uses CC BY-NC-SA 3.0, 2) KITTI-360 which uses CC-BY-NC-SA-3.0, 3) nuScenes which uses custom CC BY-NC-SA 4.0 with exceptions for startups and research and separate commercial licenses, 4) Waymo Open Dataset which uses a custom non-commercial licence.
\begin{table*}[t]
  \centering
  \resizebox{0.85\linewidth}{!}{
  \begin{tabular}{lcccccc}
    \toprule
\multicolumn{3}{c}{\textbf{Configurations}} & \multicolumn{2}{c}{\textbf{Perceptual}} & \multicolumn{2}{c}{\textbf{Statistical}}\\
\cmidrule(lr){1-3} \cmidrule(lr){4-5} \cmidrule(lr){6-7}
Backbone & Loss & RRN Sem. Input&FSVD$\downarrow$	&FPVD$\downarrow$	&JSD$\times10^{(-2)}\downarrow$	&MMD$\times 10^{(-5)}\downarrow$\\
\midrule
/ (baseline) & / &- &11.7	&13.1	&11.6	&21.5\\
\midrule
SPUNet & Welsch & -  &10.6	&12.1	&\textbf{6.7}	&\textbf{16.6}\\
SPUNet & MSE & -   &27.0	&25.7	&10.1	&22.3\\
SPUNet & Welsch & \checkmark &\underline{9.8}	&\textbf{9.7}	&\underline{9.5}	&\underline{20.9}\\
\midrule
PTV3 & Welsch &	 -          &\textbf{9.3}	&\underline{9.8}	&9.7	&22.1\\
PTV3 & MSE & -              &48.9	&50.9	&10.7	&26.1\\
PTV3 & Welsch & \checkmark  &10.2	&10.6	&9.6	&21.9\\
    \bottomrule
  \end{tabular}
  }
  \caption{\textbf{Ablation experiment} on nuScenes, including RRN backbone structure, loss function and semantic-map input to RRN.}
  \label{tab:ablation_nusc}
\end{table*}

\begin{table*}[t]
  \centering
  \resizebox{0.85\linewidth}{!}{
  \begin{tabular}{lcccccc}
    \toprule
\multicolumn{3}{c}{\textbf{Configurations}} & \multicolumn{2}{c}{\textbf{Perceptual}} & \multicolumn{2}{c}{\textbf{Statistical}}\\
\cmidrule(lr){1-3} \cmidrule(lr){4-5} \cmidrule(lr){6-7}
Backbone & Loss & RRN Sem. Input&FSVD$\downarrow$	&FPVD$\downarrow$	&JSD$\times10^{(-2)}\downarrow$	&MMD$\times 10^{(-5)}\downarrow$\\
\midrule
/ (baseline) & / &- &11.8	&12.4	&9.5	&13.4\\
\midrule
SPUNet & Welsch & -  &\underline{9.8}	&\textbf{9.9}	&7.9	&13.4\\
SPUNet & MSE & -   &24.8	&27.3	&8.7	&\textbf{11.7}\\
SPUNet & Welsch & \checkmark &\textbf{9.6}	&11.6	&7.8	&\underline{13.0}\\
\midrule
PTV3 & Welsch &	 -   &10.3	&\underline{11.4}	&\underline{7.6}	&13.6\\
PTV3 & MSE & -   &25.6	&27.8	&9.2	&15.3\\
PTV3 & Welsch & \checkmark   &11.4	&14.5	&\textbf{7.5}	&\underline{13.0}\\
    \bottomrule
  \end{tabular}
  }
  \caption{\textbf{Ablation experiment} on Waymo, including RRN backbone structure, loss function, and semantic-map input to RRN.}
  \label{tab:ablation_waymo}
\end{table*}

\begin{table}[t]
  \centering
  \resizebox{\linewidth}{!}{
  \begin{tabular}{lcccc}
    \toprule
\multirow{2}{*}{$\nu$} & \multicolumn{2}{c}{Perceptual} & \multicolumn{2}{c}{Statistical}\\
\cmidrule(lr){2-3} \cmidrule(lr){4-5}
&FSVD$\downarrow$	&FPVD$\downarrow$	&JSD$\times10^{(-2)}\downarrow$	&MMD$\times 10^{(-5)}\downarrow$\\
\midrule
/ (baseline) &18.3	&15.3	&7.1	&16.2\\
\midrule

0.01	&18.1	&14.9	&6.9	&\textbf{16.2}\\
0.05	&17.9	&14.3	&6.9	&\underline{16.3}\\
0.1	    &17.9	&\underline{13.4}	&7.0	&\underline{16.3}\\
0.5	    &\underline{17.7}	&13.9	&\underline{6.8}	&17.2\\
1.0     &\textbf{16.0}   &\textbf{13.3}	&\textbf{6.7}	&17.5\\
    \bottomrule
  \end{tabular}
  }
  \caption{\textbf{Parameter tuning experiment of $\nu$} with PTV3 on SemanticKITTI.}
  \label{tab:parameter}
  \vspace{-0.5cm}
\end{table}

\paragraph{Metric specifications.} Our performance metrics measure generation quality from two different perspectives: perceptual and distributional. 
The extended FIDs include Fr\'echet Sparse Volume Distance (FSVD) and Fre\'chet Point Voxel Distance (FPVD), which are calculated with the average bottleneck features of pre-trained MinkowskiNet \citep{choy20194d} and SPVCNN \cite{tang2020searching}, respectively. 
The JSD and MMD represent global distributional similarity in BEV space, where the point clouds are voxelized with $0.5m$ voxel size and accumulated along the z-axis into a 2D bin-count map. JSD is then defined as the Jensen-Shannon Divergence between the average BEV map distribution of two sets of point clouds. MMD represents the average of the minimum distance from generated BEV maps to ground-truth BEV maps.

\paragraph{Training cost of RRN.}
With a batch size of 60, the RRN training time is 2 hours on a single RTX 4090. In comparison, LiDM training on 4 RTX 4090s costs 36 hours for 150 epochs. 
Therefore, RRN is much more training-efficient than LiDAR RV diffusion.

\subsection{Additional Experiments}
\label{ap:exp}

\begin{table*}[t]
    \centering
    \resizebox{\linewidth}{!}{
    \begin{tabular}{l l c c c c c c c}
    \toprule
    \textbf{RRN Training Data} & \textbf{Test Data} & \textbf{C-D} $\downarrow$ & \textbf{F-score} $\uparrow$ & \textbf{RMSE} $\downarrow$ & \textbf{MedAE} $\downarrow$ & \textbf{LPIPS} $\downarrow$ & \textbf{SSIM} $\uparrow$ & \textbf{PSNR} $\uparrow$ \\
    \midrule
    Val (3 frames $\times$ 7 scenes) & Val (1 frame $\times$ 7 scenes) & 0.0557 & 0.966 & 1.64 & 0.00647 & 0.0233 & 0.978 & 33.9 \\
    None (GS-LiDAR original) & Val (1 frame $\times$ 7 scenes) & 0.0560 & 0.963 & 1.64 & 0.00832 & 0.0231 & 0.978 & 34.0 \\
    Train (46 frames $\times$ 7 scenes) & Val (4 frames $\times$ 7 scenes) & 0.0473 & 0.966 & 2.46 & 0.0171 & 0.0651 & 0.776 & 30.3 \\
    None (GS-LiDAR original) & Val (4 frames $\times$ 7 scenes) & 0.0480 & 0.961 & 2.21 & 0.0156 & 0.0518 & 0.818 & 31.2 \\
    \bottomrule
    \end{tabular}
    }
    \caption{RRN performance on GS-LiDAR reconstructed LiDAR data under different training and test data configurations.}
    \label{tab:rrn_results}
\end{table*}

\paragraph{Ablation on nuScenes.} We ablate various components of the L3DR framework in nuScenes in Tab.~\ref{tab:ablation_nusc}. We calculated metrics with voxelized GT instead of original GT for fast evaluation, which results in slightly different figures. Specifically, applying MSE Loss deteriorates all metrics, which is also consistent with the conclusion on KITTI. The best overall performance again falls on SPUNet with Welsch loss and Segmentation input, with $9.8$ FSVD, $9.7$ FPVD, $9.5\times 10^{-2}$ JSD, and $20.9\times 10^{-5}$ MMD. On the other hand, adding semantic maps for PTV3 does not provide such consistent performance boosts, making it a suboptimal choice. Additionally, substituting the 3D RRN with a 2D UNet again deteriorates all performance metrics. We conclude that SPUNET and Welsch Loss are the best setup for nuScenes.

\paragraph{Ablation on Waymo.} We ablate various components of the L3DR framework in Waymo Open Dataset in Tab.~\ref{tab:ablation_waymo}. We calculated metrics with voxelized GT instead of original GT for fast evaluation, which results in slightly different figures. Specifically, applying MSE Loss deteriorates most performance metrics compared to baseline, where the FSVD and FPVD all doubled for SPUNet and PTV3. The best overall performance attributes to SPUNET with Welsch loss, with $9.8$ FSVD, $9.9$ FPVD, $7.9\times 10^{-2}$ JSD, and $13.4\times 10^{-5}$ MMD. Additionally, using a 2D Image Unet as RRN again brings considerable degradation. We conclude that SPUNET and Welsch Loss are the best setup for Waymo.

\paragraph{Parameter choice.} The choice of the width parameter $\nu$ of the Welsch Loss is discussed in Tab.~\ref{tab:parameter}. We observe contrary trends between the first 3 metrics and the last metric, where FSVD, FPVD and JSD improve with larger $\nu$ while MMD deteriorate in contrast. This trend reveals an intriguing property of the metrics: MMD is a local metric accumulated from individual point cloud pairs, which favors an average shape to match all possible point clouds. In comparison, FID and JSD are global metrics calculated over all point clouds, which favors perceptual realism and coherence. Therefore, we interpret that larger $\nu$ means individual point clouds are highly rectified and further from the average shape, which results in higher MMD, while larger $\nu$ offers better realism which is preferable for perceptual FID scores. Therefore, we choose $\nu=0.5$ for good FID and JSD performance, while keeping an acceptable performance on MMD.

\paragraph{Compatibility with reconstruction methods.} Table \ref{tab:rrn_results} shows RRN experiments on very limited reconstruction data of GS-LiDAR on KITTI-360. Since GS-LiDAR reconstruction fails on 3 out of 10 scenes due to VRAM OOM, the experiments were conducted over 7 scenes only (all training data has been augmented with z-axis rotation, flipping, and scaling, the network was trained for 1000 epochs, and the width parameter 
 is set to 0.2 due to shrinked artifact sizes). Specifically, we trained two RRNs by using the GS-LiDAR validation frames (4/50 frames in a scene) and the GS-LiDAR training frames (46/50 frames in a scene), respectively. While training RRN on the validation frames, we use the first 3 validation frames of each scene for training and the last validation frame for testing. While training RRN on the training frames, we test the trained model on all 4 validation frames of the scene. We can see that the validation-frame trained RRN performs slightly better than the original GS-LiDAR reconstructed frames across most performance metrics, while the training-frame trained RRN performs worse on multiple metrics besides Chamfer Distance and F-Score, confirming that reconstruction artifacts are only present in the validation frames. These new experiments over limited training data suggest that RRN could be incorporated as a promising addition to the reconstruction methods like GS-LiDAR. Experiments on larger-scale reconstruction data counts as future work.

\subsection{Limitations}
\label{ap:limitation}
While our proposed 3D rectification framework significantly improves the geometric fidelity of diffusion-generated LiDAR outputs, several limitations remain. First, the method depends on the initial quality of the range-view depth prediction. Severe RV artifacts that exceed the restoration capacity of the RRN may still leak into the rectified point clouds. Second, although our approach suppresses RV artifacts such as wavy surfaces, round corners, and depth bleeding, our pipeline does not explicitly model such artifacts, which may still result in room for improvement, especially in complex scenes. Finally, the L3DR framework relies on high-quality semantic-conditioned LiDAR diffusion model for training data generation. The general application of the L3DR framework in indoors or objects depends on the availability and quality of conditional diffusion models in these fields.

\begin{table*}
    \centering
    \small
    \setlength{\aboverulesep}{0pt} 
\setlength{\belowrulesep}{0pt} 
    \resizebox{\linewidth}{!}{
    \begin{tabular}{l|c|c}
    \toprule
         Method & Benefits &	Limitations \\
         \midrule
\multirow{3}{*}{\begin{tabular}{@{}l@{}}RV Diffusion \\(e.g., R2DM, LiDM, ours)\end{tabular}}	&• Lightweight and low-cost generation & \multirow{3}{*}{\begin{tabular}{@{}c@{}}• Single-frame generation Geometry artifacts\\• Moderate Gen2Real gap))\end{tabular}}\\
&• Broad condition compatibility & \\
&• Natural modeling of LiDAR’s projective structure	& \\
\midrule
\multirow{3}{*}{\begin{tabular}{@{}l@{}}BEV Diffusion \\(e.g., UltraLiDAR)\end{tabular}}	&\multirow{3}{*}{\begin{tabular}{@{}l@{}}• Supports manipulation of BEV grid\\• Local spatial control with partial point input\end{tabular}} &• Poor projection realism \\
& 	& • Limited multimodal control \\
& & • Still single-frame\\
\midrule
\multirow{3}{*}{\begin{tabular}{@{}l@{}}4D Occupancy Diffusion  \\(e.g., UniScene, OccSora)\end{tabular}}	&• Spatiotemporal coherence 	&• High computational burden \\
&• Multimodal output such as image and occupancy &• Occupancy control is unnatural to human users \\
&• Spatio-temporal consistency &• Hard-hit approximated with volumetric rendering\\
\midrule
Simulators (e.g., CARLA)	&• Accurate projection model • Expanding asset libraries	&• Large Sim2Real gap • Difficult to scale up\\
\midrule
\multirow{3}{*}{\begin{tabular}{@{}l@{}}Reconstruction Methods  \\(e.g., LiDiff, NeRF-LiDAR)\end{tabular}}	&\multirow{3}{*}{\begin{tabular}{@{}l@{}}• High geometric fidelity  \\• Fast inference (esp. 3DGS)	\end{tabular}}&• Requires per-scene training\\
&&• Limited layout diversity \\
&&• Large storage and pre-processing overhead\\
\bottomrule
    \end{tabular}
    }
    \caption{Comparative overview of existing genres for LiDAR point cloud generation.}
    \label{tab:comparative_overview}
\end{table*}

\begin{figure*}[h!]
  \centering
  \includegraphics[width=0.9\linewidth]{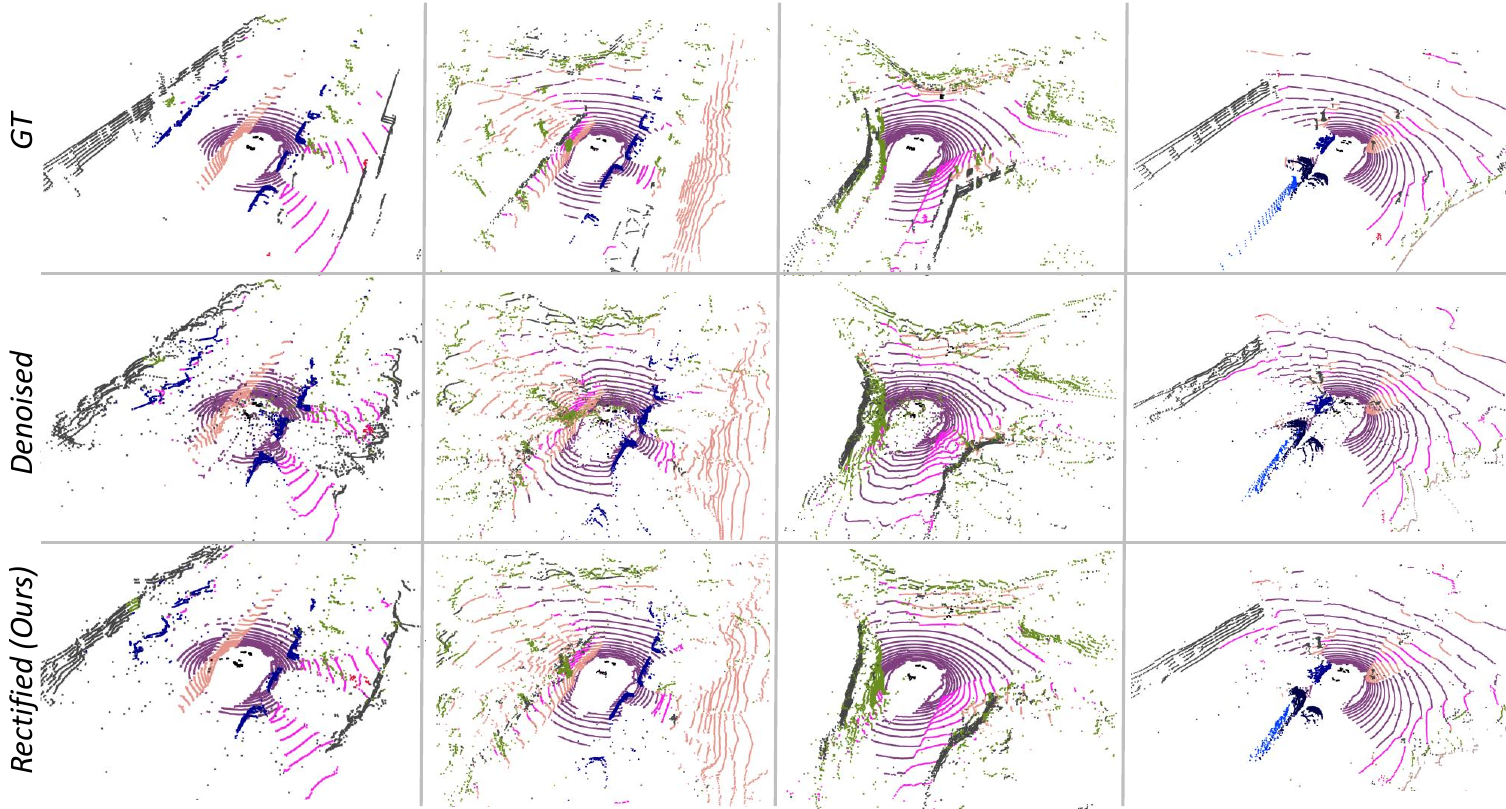}
  \caption{\textbf{Visualization} of conditional generation on nuScenes. Our results exhibit better geometry-realism such as sharp borders, flat surfaces, less noise points and reduced depth bleeding.}
  \label{fig:nuscenes}
\end{figure*}

\begin{figure*}
  \centering
  \includegraphics[width=0.9\linewidth]{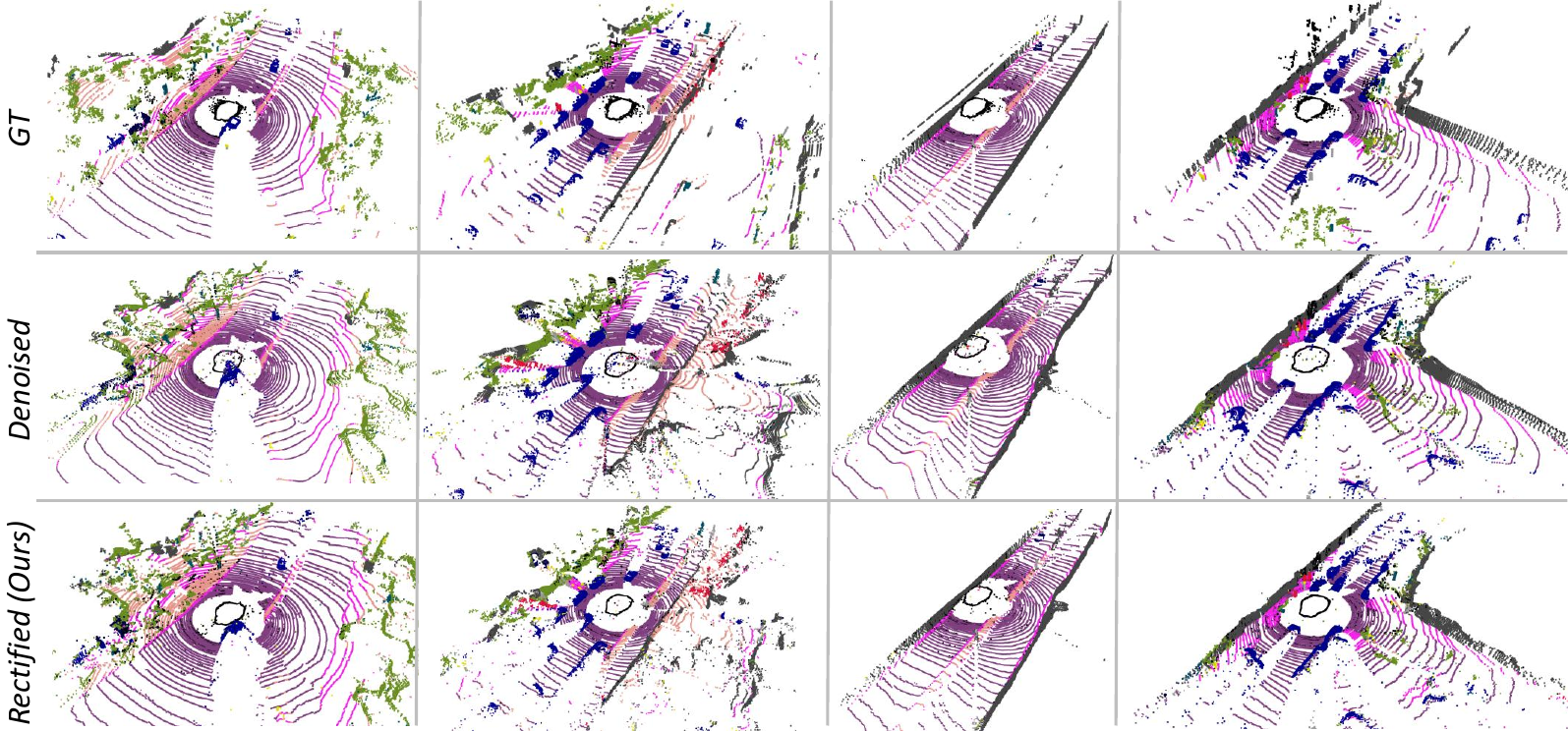}
  \caption{\textbf{Visualization} of conditional generation on Waymo Open Dataset. Our results exhibit better geometry-realism with reduced depth bleeding, flattened surfaces, rectified corners and reduced noise points.}
  \label{fig:waymo}
\end{figure*}

\subsection{Discussions}
\paragraph{Integraton of two-stage training into one.} We think it is inefficient and damaging to integrate the two-stage training into a single-stage joint training. Compared with point clouds generated by a converged diffusion model, point clouds generated during training have a compound of residual Gaussian noise and RV artifacts. Such noisy point clouds offer little help to the RV artifact regression. They also make RNN training challenging, leading to a degraded RNN model and further degraded generation performance.

\paragraph{Regression of all error regions rather than only regressing RV artifacts.} We highlight that RV artifacts are errors with consistent patterns across different diffusion models and datasets, while high-bias errors are not even consistent in the same model and same dataset with different random seeds. This is because RV artifacts are caused by inherent property that a 2D model tends to attend to global structures and ignore small errors as discussed in Section \ref{sec:pilot}, which is shared across all models despite nuances in implementation. On the other hand, high-bias errors are introduced by insufficient geometric constraints in large chunks of the semantic map control input, which can have multiple reasonable interpretations and each diffusion run could randomly choose from one of them. Therefore, it shares significantly less general value to regress all errors instead of RV artifacts only, which makes the model specific to that run rather than generally applicable to other models.
\label{sec:regress_all}

\paragraph{Comparative overview of RV in all LiDAR generation genres.} To contextualize our method, we provide a comparative overview of representative LiDAR generation paradigms in Table \ref{tab:comparative_overview}. Unlike BEV or volumetric (4D occupancy), range-view (RV) LiDAR generation natively handles self-occlusion and projection realism, which are crucial in LiDAR perception. Additionally, image-based conditioning is far more accessible and generalizable than occupancy grids or partial point clouds, thus positioning RV diffusion as a lightweight and highly versatile choice. Further, RV diffusion enjoys a higher level of versatility compared to driving simulators or reconstruction methods. At the other end, RV LiDAR generation does potentially produce geometric artifacts, but this is rather a common issue across different LiDAR generation paradigms. Hence, the conclusion is that each generation paradigm has its unique pros and cons and application scenarios. RV diffusion possesses superior versatility and computational efficiency which is a promising research direction in our humble opinion.
We aim to address the key bottleneck of RV diffusion, i.e., geometry realism, and propose an effective solution through 3D residual regression. We believe that our work meaningfully advances the line of RV-based LiDAR generation research.

\paragraph{Directly diffuse in 3D space instead of RV.} Diffusing in 3D from scratch cannot generate authentic global structures as reported in Nunes et al [1]. Our preliminary experiments on increasing diffusing steps in 3D also show receding performance as more diffusion steps are involved. Our understanding is that 2D diffusion models are good at reducing global noise instead of local noises or range view (RV) artifacts, while 3D networks are good at handling local geometry instead of global noises. We will discuss this issue in the revised manuscript.

\paragraph{Multi-step rectification with RRN.} Empirically, we have tested an alternative version of multi-step RRN where each step are trained to regresses the residual error of the previous step. However, experiments reveal receding performance as more steps are added. We interpret this phenomenon as increasingly harder and misaligned learning objectives hinder the learning of of a consistent strategy. For example, 90\% of the error may come from RV artifacts in the first round, the second round may still have 30\% unsolved RV artifacts and 60\% residual regression error of the first step, and the third round still contains raw RV artifatcs, artifacts unsolved in the first step, and those unsolved in the second step. More steps indicate a more convoluted error pattern which hinders accurate recognition and cancellation of these residuals, which damages the overall performance.

\paragraph{Difference between RRN and post-processing networks in 3DGS.} We point out that in reconstruction models like NeRF-LiDAR, LiDAR4D, GS-LiDAR, and LiDAR-RT, the post-processing by U-Net serves to regress raydrop possibilities for reconstruction. It is formulated as a binary classification task and trained with BCE loss. Differently, our 3D RRN is a novel design to LiDAR diffusion aiming at removing RV artifacts present in 3D coordinates. It is formulated as a vector regression task and trained with MSE/Welsch loss. Hence, both approaches share little similarity besides using U-Nets for post-processing.

\end{document}